\documentclass{article}

\usepackage{times}
\usepackage{graphicx} 
\usepackage{natbib}
\usepackage{xspace}

\usepackage{booktabs}

\usepackage{hyperref}
\usepackage{cleveref}

\usepackage{subcaption}

\usepackage[accepted]{icml2015}

\usepackage{amsmath}
\usepackage{amsthm}
\usepackage{amssymb}

\usepackage{tikz}
\usepackage{pgfplots}
\usepackage{pgf}

\icmltitlerunning{Scalable Variational Inference in Log-supermodular Models}

\newcommand{\x}{\mathbf{x}}
\newcommand{\y}{\mathbf{y}}
\newcommand{\s}{\mathbf{s}}
\newcommand{\z}{\mathbf{z}}
\newcommand{\m}{\mathbf{m}}
\newcommand{\w}{\mathbf{w}}

\newcommand{\p}{\mathbf{p}}
\newcommand{\q}{\mathbf{q}}
\newcommand{\R}{\mathbb{R}}      \newcommand{\Ent}{\mathbb{H}}  \newcommand{\Z}{\mathcal{Z}}  \renewcommand{\L}{\mathcal{L}}      \newcommand{\B}{\mathcal{B}}      
\newcommand{\seq}{\subseteq}  
\newcommand{\algo}{{\textsc{L-Field}}\xspace}

\newcommand{\Q}{\mathcal{Q}}  
\newcommand{\OPT}{\textsc{OPT}}

\newcommand{\Div}[3]{{D_{#1}({#2} \,\|\, {#3})}}  
\DeclareMathOperator{\minimize}{minimize}

\newtheorem{lemma}{Lemma}
\newtheorem{defn}{Definition}
\newtheorem{thm}{Theorem}

\newtheorem{coro}{Corollary}

\begin{document}

\twocolumn[
\icmltitle{Scalable Variational Inference in Log-supermodular Models}

\icmlauthor{Josip Djolonga}{josipd@inf.ethz.ch}
\icmlauthor{Andreas Krause}{krausea@ethz.ch}
\icmladdress{Department of Computer Science, ETH Zurich}

\icmlkeywords{}

\vskip 0.3in
]

\begin{abstract}
    We consider the problem of approximate Bayesian inference in log-supermodular models.
    These models encompass regular pairwise MRFs with binary variables, but allow to capture high-order interactions, which are intractable for existing approximate inference techniques such as belief propagation, mean field, and variants.
    We show that a recently proposed variational approach to inference in log-supermodular models --\algo-- reduces to the widely-studied minimum norm problem for
    submodular minimization. This insight allows to leverage powerful existing tools, and hence to solve the variational problem orders of magnitude more efficiently than previously possible.
    We then provide another natural interpretation of \algo, demonstrating that it \emph{exactly} minimizes a specific type of R\'enyi divergence measure. This insight sheds light on the nature of the variational approximations produced by \algo.
    Furthermore, we show how to perform parallel inference as message passing in
    a suitable factor graph at a linear convergence rate, without having to sum up over
    all the configurations of the factor. Finally, we apply our approach to a challenging image segmentation task. Our experiments confirm scalability of our approach, high quality of the marginals, and the benefit of incorporating higher-order potentials.
\end{abstract}

\makeatletter{}\section{Introduction}

Performing inference in probabilistic models is one of the central challenges in machine
learning, providing a foundation for making decisions with uncertain data.
Unfortunately, the general problem is intractable and one must resort to approximate inference
techniques. The importance of this problem is witnessed by the amount of interest it has
attracted in the research community, which has resulted in a large family of approximations,
most notably the mean-field~\cite{wainwright-ft} and belief propagation~\cite{pearl1986fusion}
algorithms and their variants. One major drawback of these and many other techniques is the
exponential dependence on the size of the largest factor which restricts the class of
models one can use. In addition, these methods generally involve non-convex objectives,
resulting in local optima (or even non-convergence).

We consider the problem of inference in distributions over sets, also known as point
processes. Formally, we have some finite ground set $V$ and a measure $P$ that
assigns some probability $P(A)$ to every subset $A\seq V$. We would like to point out that we can
equivalently see such distributions as being defined over $|V|$ Bernoulli random variables
$X_i\in\{0,1\}$, one for every element in the ground set $i\in V$ indicating if element $i$
has been selected. As a concrete example showing this equivalence consider the task
of image segmentation in computer vision, where one wants to separate the foreground from the
background pixels. Traditionally, one defines one random variable $X_p\in\{0,1\}$ for each pixel
$p$ indicating if the pixel is in the foreground or the background.
We can also isomorphically treat the distribution as being defined over \emph{subsets} of the set
of all pixels $V$.
In this case, for any subset $A\seq V$ the quantity $P(A)$ is the probability of pixels $A$ belonging to
the foreground.
In the remaining of the paper we will employ this latter view of distributions over sets.
The additional assumption that we make is that the distribution is \emph{log-supermodular}, i.e.\ can
be written as $P(A)=\frac{1}{\Z}\exp(-F(A))$, where $F$ is some {\em submodular} function.

\paragraph{Related work.}
Submodular functions are a family of set functions exhibiting a natural diminishing returns
property, originating first in the field of combinatorial optimization~\cite{edmonds1970submodular}.
They have been applied to many problems in machine learning, including clustering~\cite{qclustering}, 
variable selection~\cite{krause05near}, structured norms~\cite{bach2010structured}, dictionary learning~\cite{cevher11greedy}, etc. Submodular functions have huge implications for the tractability of (approximate) optimization, akin to convexity and concavity in continuous domains. While the major emphasis has consequently been on optimization,  submodular functions can be also employed to define probabilistic models. Special cases include Ising models used in computer models and the determinantal point process (DPP)~\cite{kulesza2012determinantal} used for modeling diversity.
Alas, submodularity does not render the inference problem tractable, which
remains extremely difficult even for the Ising model~\cite{goldberg2007complexity,jerrum1993polynomial} which has only pairwise interactions.  

\looseness -1 The study of approximate Bayesian inference in general log-supermodular models has been recently initiated by \citet{djolonga14variational}. They provide a general variational approach --\algo -- that optimizes bounds
on the partition function via the differentials of submodular functions. While their approach leads to optimization problems that can be solved exactly in polynomial time for arbitrary high order interactions, presently the approach is slow, and impractical for large scale inference tasks such as those arising in computer vision.

\paragraph{Our contributions.} We improve over their result in several ways. First, by showing an equivalence of \algo
with a classical problem in submodular minimization -- the minimum norm point problem -- we obtain access to a large family of
specially crafted algorithms that can handle models with very large numbers of variables. In the experimental
section we indeed perform inference over images, which have hundreds of thousands of variables. This insight also implies, for example, that the approximation agrees on the mode of the distribution, hence the MAP problem is solved for free.
Secondly, by establishing another important connection, namely to a specific type of information divergence, we shed light on the type of approximations that result from this method. Thirdly, we show how special structure of many real-world log-supermodular models (such as those in image segmentation with high-order potentials) enable a highly efficient parallel message passing algorithm that converges to the global optimum at a linear rate. Lastly, we perform extensive experiments on a challenging image segmentation task, demonstrating that our approach is scalable, provides more accurate marginals than existing techniques, and provides evidence on the effectiveness of models using high-order interactions.
 
\makeatletter{}\section{Background: Submodularity and log-supermodular models}\label{sec:submodularity}

Formally, a function $F:2^V\to\R$ is said to be submodular if for any pair of sets $A\seq B$ and
$x\notin B$ it holds that
\[
    F(\{x\} \cup A) - F(A) \geq F(\{x\} \cup B) - F(B).
\]
In other words, the \emph{gain} of adding any element $x$ decreases as the context grows, which is
the diminishing returns property already mentioned. Additionally, without any loss of
generality we assume that $F$ is normalized so that $F(\emptyset)=0$.
We will consider Gibbs distributions that arise from these models, more
specifically probability measures of the form
\[
    P(S) = \frac{1}{\Z} \exp(-F(S)),
\]
for some submodular $F:2^V\to\R$. These models are called \emph{log-supermodular} or
\emph{attractive} for reasons explained below.  
\paragraph{Examples.} A typical example of such models is the regular Ising model, which can be
used for the image segmentation task from the introduction. Continuing with that
example, we define a set of edges $E$ that connect neighboring pixels, and for every pair of
neighbors $\{p, p'\}$ we specify a weight $w_{\{p,p'\}}\geq 0$ that quantifies their
similarity. To model the preference of neighbors to be assigned to the same segment,
we use the {\em cut function}
\[
    \forall A\seq V \colon F(A) = \sum_{\{p, p'\}\in E} 1_{|A\cap\{p, p'\}|=1} w_{\{p, p'\}}.
\]
Hence, if we place two neighboring pixels $p$ and $p'$ in different segments, we will
cut the edge $\{p, p'\}$ and be ``penalized'' by the corresponding weight, which explains
the attractive behavior of the model. We can go one step further and define regions
$P_i\seq V$ which we would prefer to be in the same segment. One strategy to generate the
regions, used by~\citet{kohli2009robust} is to generate superpixels, as illustrated on Figure~\ref{fig:sp}.
We can then modify the model to incorporate these higher order potentials by adding
terms of the form $\phi(|P_i\cap A|/|P_i|)$ for some concave function $\phi$. As a concrete
example, consider $\phi(z)=z(1-z)$, which assigns a value of 0 when the pixels in the superpixel
are in the same segment, and assigns a larger penalty otherwise, which is maximal when the
pixels are equally split between the two segments.
\begin{figure}
    \centering
    \includegraphics[width=0.32\textwidth]{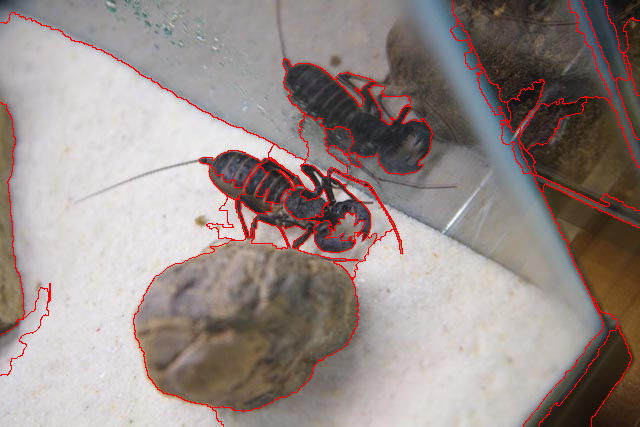}
    \caption{Generated superpixels to be used as attractive higher order potentials for encouraging label consistency.}
    \label{fig:sp}
\end{figure}

\paragraph{Modular functions.} The simplest family of submodular functions are modular functions, which can be seen as the
discrete analogue of
linear functions.
Namely, a function $s:2^V\to\R$ is said to be modular if for all $A\seq V$ it holds that
$s(A) = \sum_{i\in A} s(\{i\})$.
The family of distributions that arise from these functions are exactly the family of completely factorized distributions\footnote{Because we use Gibbs distributions, note that they can not assign zero probabilities.},
because
\[
    P(S) \propto \exp(-s(S)) = \prod_{i\in S} \exp(-s_i).
\]
As evident from their definition, modular functions are uniquely defined through the quantities
$s(\{i\})$ for all $i\in V$. It is very useful to view modular functions as
vectors $\s\in\R^V$ with coordinates $s_i = s(\{i\})$.
\paragraph{Submodular polyhedra.} There are several polyhedra that contain some of these modular functions (in their vectorial representation)
that we will make use of. More specifically, we are interested in the submodular polyhedron $P(F)$ and the
base polytope $B(F)$, which are defined as
\begin{align}
    P(F) &= \{ \s\in\R^V \mid \forall A\seq V \colon s(A) \leq F(A) \}, \\
    B(F) &= P(F)\cap \{ \s\in\R^V \mid s(V) = F(V) \}.
\end{align}
In other words, $P(F)$ is the set of all modular \emph{lower bounds} of the
function $F$, while $B(F)$ adds the further restriction that the bound must be tight
at the ground set $V$.
It can be shown that these polyhedra are not empty and their geometry is also
well understood~\cite{fujishige-book,fbach-ft}. Moreover, what is especially surprising, is that
even though $B(F)$ is defined with exponentially many inequalities, we can optimize linear functions
over it in~$O(|V|\log|V|)$ time with a simple greedy strategy~\cite{edmonds1970submodular}.

\paragraph{MAP estimation and the minimum norm point.} A very natural question that arises for any probabilistic model is that of finding a
MAP configuration, which for log-supermodular distribution amounts to minimizing the
function $F$. This is a problem that has
been studied in much detail and has resulted in numerous approaches. The fastest known
combinatorial algorithm due to~\citet{orlin2009faster} has a bound of $O(n^6 + \tau n^5)$, where $\tau$
is the cost of evaluating the function, and can be prohibitively expensive to run for larger
ground sets.
An algorithm that performs better in practice, but only has a pseudopolynomial running time
guarantee \cite{chakrabarty2014provable}, is the Fujishige-Wolfe
algorithm~\cite{fujishige1980lexicographically}. This method approaches the problem by solving
the following convex program, known as the \emph{minimum norm problem}.
\begin{equation}\label{prob:mn}
    \underset{\s\in B(F)}{\minimize}\, \|\s\|^2.
\end{equation}
One can extract the solution to the submodular minimization problem from the
solution to the above problem by thresholding, which is formalized in the following theorem.
\begin{thm}[\citet{fujishige-book}]\label{thm:fuji}
    If $\s^*$ is the optimal solution to problem~\eqref{prob:mn}, define the following
    sets
        \begin{align*}
            A_- &= \{ v \mid v\in V \textrm{ and } s^*_v < 0 \}, \textrm{and}\\
            A_0 &= \{ v \mid v\in V \textrm{ and } s^*_v \leq 0 \}.
        \end{align*}
    Then $A_-$ and $A_0$ are the unique minimal and maximal minimizers of $F$.
\end{thm}
 
\makeatletter{}\section{Variational inference with \algo}
The main barrier to performing inference in log-supermodular models is the computation
of the normalizing factor $\Z$, also known as the partition function in the statistical
physics literature. We cannot compute it directly as we have to sum up over all
$S\seq V$, so we have to use approximative techniques. One common approach is to define
an optimization problem over some variational parameter $\q$, so that we can compute the
quantity of interest by optimizing this problem.

We now review the variational approximation technique recently introduced by~\citet{djolonga14variational}.
Their method relies on two main observations: \emph{(i)} modular functions have analytical
log-partition functions and \emph{(ii)} submodular functions can be lower-bounded by modular
functions. The main idea is the following: if it holds that $\forall A\seq V\colon s(A)\leq F(A)$, then
it will certainly be the case that
\[
    \log\sum_{A\seq V} e^{-F(A)} \leq \log\sum_{A\seq V} e^{-s(A)} = \sum_{i\in V} \log(1+e^{-s_i}).
\]
We thus have a family of variational upper bounds on the partition function parametrized by the modular
functions $\s$, over which we can optimize to minimize the right hand side of the inequality.
As shown by~\citet{djolonga14variational} this variational problem can be reduced to the
following convex separable optimization problem over the base polytope
\begin{equation}\label{prob:varinf}
    \underset{\s\in B(F)}{\minimize} \sum_{i\in V} \log(1+\exp(-s_i)).
\end{equation}
This problem -- \algo\ -- can be then solved using the divide-and-conquer
algorithm~\cite{fbach-ft,jegelka2013reflection} by solving at
most~$O(\min\{ |V|, \log\frac{1}{\epsilon} \})$ MAP problems,
where $\epsilon$ is the tolerated error on the marginals. It can be also approximately
solved using the Frank-Wolfe algorithm at a convergence rate of $O(1/k)$. While these results establish tractability of the variational approach, in general solving even one MAP problem requires submodular minimization -- an expensive task, and repeated solution may be too costly. Convergence of the Frank-Wolfe method is empirically slow. 

\section{\algo $\equiv$ Minimum norm point.} Our first main contribution is the following, perhaps surprising, result:
\begin{thm}\label{thm:equiv}
    Problems~\eqref{prob:varinf} and~\eqref{prob:mn} have the same solution.
\end{thm}
The proof of this theorem (given in the appendix) crucially depends on the peculiar characteristics of the base polytope. Similar results have been shown (for other objectives) by~\citet{nagano2012equivalence}. This theorem has three immediate, extremely important consequences. First, since the minimum-norm point approach is often the method of choice for submodular minimization anyway, this insight reduces the cost from solving many MAP problems to a {\em single} minimum norm point problem, which leads to substantial performance gains -- a factor of $O(|V|)$ if we seek the optimal variational solution!
Secondly, given this equivalence and Theorem~\ref{thm:fuji}, we can immediately see that
we can in fact extract the MAP solution by thresholding the marginals at $1/2$.
\begin{coro}\label{coro:map}
    We can extract the unique minimal and maximal MAP solutions by thresholding the
    optimal marginal vector at $1/2$.
\end{coro}
Thus, the \algo approach results in the {\em exact} MAP solution in addition to approximate marginals and an upper bound on the partition function.
Thirdly, since the minimum norm point problem is well studied, faster algorithms for important special cases become available. In particular, in \S 6, we demonstrate how certain types of log-supermodular distributions enable extremely efficient parallel inference.
 
\makeatletter{}\section{The divergence minimization perspective}\label{sec:divergence}
\looseness -1 The \algo approach attacks the partition function directly. One can of course employ the factorized distribution parametrized by the minimizer $\s^*$ of the upper bound to obtain approximate marginals. However, it is not immediately clear what properties the resulting distribution has, apart from agreeing on the mode
(as shown by Corollary~\ref{coro:map}).
To this end, we turn to the theory of divergence measures as that will enable us to
understand the solutions preferred by the method. Divergence measures are functions
$\Div{}{P}{Q}$ of two probability distributions $P$ and $Q$ that quantify the degree of
dissimilarity between the arguments. Once we have picked a divergence measure $D$, we are
interesting in minimizing $\Div{}{P}{Q}$ among some set of approximative distributions $Q\in\Q$.
The family which is of particular interest to us is that of completely factorized
distributions that assign positive probabilities, which we now formally define.
\newpage
\begin{defn}
We define the set $\Q$ of completely factorized positive distributions as
\[
    \Q = \{ Q \mid Q(S) \propto \prod_{i\in S}\exp(-q_i) \; \textrm{ for some } \q\in\R^V \}.
\]
\end{defn}
There are many choices for a divergence measure, the most prominent examples being the KL-divergence
and the family of R\'enyi divergences~\cite{renyi}.
Of particular interest for our analysis  is the special infinite order of the R\'enyi divergence,
defined as follows:
\begin{defn}[\citet{van2012r}]
Define the R\'enyi divergence of infinite order between $P(S)$ and $Q(S)$
\begin{equation}
    \Div{\infty}{P}{Q} = \log\sup_{S\seq V} \frac{P(S)}{Q(S)}.
\end{equation}
\end{defn}
In the terminology of~\citet{minka2005divergence} we can see that the $D_{\infty}$ divergence is
\emph{inclusive}, which means that it would try to ``cover'' as much as possible from the distribution: The variational approximation is conservative in the sense that it attempts to spread mass over all sets that achieve substantial mass under the true distribution.
As we now show, it turns out that when we minimize this divergence for log-supermodular distributions
we can focus our attention only on some specific factorized distributions.
\begin{lemma}\label{lem:gb}
When $P$ is log-supermodular, to solve $\minimize_{Q\in\Q} \Div{\infty}{P}{Q}$
we have to only optimize over modular functions $q$ that are global lower bounds of $F$.
\end{lemma}
What this lemma essentially says, is that a minimizing distribution $\q^*$ can be always
found in $P(F)$. This result also implies the central result of this section, that
the variational approach we have considered essentially minimizes the infinite divergence.
\begin{thm}\label{thm:div-eq}
When $P$ is log-supermodular, the problem
$\minimize_{Q\in\Q} \Div{\infty}{P}{Q}$ is equivalent to problem~\eqref{prob:varinf}.
\end{thm}
This theorem has the following immediate consequence:
\begin{coro}\looseness -1 For log-supermodular models, problem $\minimize_{Q\in\Q} \Div{\infty}{P}{Q}$ is
polynomial-time tractable via $O(|V|)$ MAP (submodular minimization) problems.
\end{coro}
Hence, any log-supermodular distribution has the property that we can find the
closest factorized distribution to it w.r.t.~this specific divergence in polynomial time,
irrespective of whether the distribution factorizes into smaller factors or not.
We would like to point out that the above criterion does not necessarily hold in general
for non-log-supermodular distributions, which we formally show.
\begin{lemma}\label{lem:counterexample}
\Cref{lem:gb} does not hold for general point processes. Specifically, there exists
a log-submodular counter example.
\end{lemma}
The proofs of all claims are provided in the supplemental material.
 
\makeatletter{}\section{Parallel inference for decomposable models}
Very often the submodular function $F$ has structure that one can exploit to procure faster
inference algorithms. In particular, the function often {\em decomposes}, i.e.,
can be written as a sum of (simpler) functions as
\[
    F(S) = \sum_{i=1}^R F_i(S\cap V_i),
\]
where $F_i : 2^{V_i}\to\R$ are submodular functions with ground sets $V_i$.
This setting has been considered, e.g., by~\citet{stobbe10efficient} and~\citet{jegelka2013reflection}.
The decomposition implies that the corresponding distribution factorizes as follows
\begin{equation}
    P(S) \propto \prod_{i=1}^R \exp(-F_i(S\cap V_i)).
\end{equation}
In fact, the examples we discussed in \S\ref{sec:submodularity} both have this form, factorizing either into pairwise potentials or into the potentials defined by the superpixels. Such models can be naturally represented via a factor graph $G$ that has as nodes the union
of the ground sets $V_i$ and the factors $F_1,\ldots,F_R$. We then add edges $E$ in a bipartite
way by connecting each factor $F_i$ to the elements $V_i$ that participate in it
(e.g. $F_i$ is connected to $v$ iff $v\in V_i$). For any node $w$ in the graph (either
a factor, or variable node), we will denote its neighbors by $\delta(w)$. 

When the function enjoys such a decomposition, the base polytope can be written as the
Minkowski sum of the base polytopes of the summands, or formally
\footnote{If $v\notin V_i$, then the elements from $B(F_i)$ are treated as having a zero
for that coordinate.}
\[
    B(F) = \sum_{i=1}^R B(F_i).
\]
Hence, the minimum norm problem~\eqref{prob:mn} that we are interested in can be rewritten
as the following problem.
\[
    \underset{\mathbf q_i\in B(F_i)}{\minimize}\sum_{v\in V} (\sum_{F_i\in\delta(v)} q_{i,v})^2.
\]
In the following, we discuss two natural message passing algorithms exploiting this structure.

\paragraph{Expectation propagation.} 
A very natural approach would be to perform block coordinate descent one factor at a time.
If we look through the lens of divergence measures, as introduced in \S 5,
we can make a clear connection to (a variant of) {\em expectation propagation}\footnote{Typically, expectation propagation is defined w.r.t.~the KL-divergence.}, the message passing approach of~\citet{minka2005divergence} specialized to
minimizing the divergence $\Div{\infty}{P}{Q}$, which we now briefly describe.
The main idea is to approximate each factor $\exp(-F_i(S\cap V_i))$ with a completely
factorized distribution $Q_i(S)\propto\exp(-q_i(S))$, such that the product
$\prod_{i=1}^R Q_i$ is a good approximation to the true distribution in terms of the given
divergence. Then, we optimize iteratively using the following procedure.
\begin{enumerate}
    \item Pick a factor $F_i$.
    \item Replace the other factors $F_j$ for $j\neq i$ with their approximations
        $Q_j$ and minimize
        \[
            \Div{\infty}{\frac{1}{\Z_i}\exp(-F_i(S))\prod_{j\neq i}Q_j}{\prod_{j=1}^R Q_j}
        \]
        over the factorized approximation $Q_i$.
\end{enumerate}
In other words, we choose a factor and minimize the infinite divergence for that factor, but instead of
using the true factors $\exp(-F_j(S))$ for $j\neq i$, we replace them with their current
modular approximations $Q_j$.

\paragraph{A parallel approach.}
One downside of the approach presented above is that it has to be applied {\em sequentially}, i.e., one factor has to be updated at a time to ensure convergence.  An alternative is to apply an approach used by~\citet{jegelka2013reflection}, which
allows to perform message passing {\em in parallel} without losing the convergence guarantees.
\citet{jegelka2013reflection} assume that all factors depend on {\em all} variables
(i.e.\ $V_i=V$). In the following, we generalize their setting in order to allow $V_i\neq V$. By changing the dual problem they consider (shown in detail in the appendix) we arrive at a form that is more natural to our setting and can be seen as performing message passing in the factor graph.
To describe the messages, we have to define the following pair of norms that arise from the
structure of the factor graph.

\begin{defn}\label{defn:norms}
    For any $\x_S\in\mathbb{R}^S$, where $S\seq V$, we define the
    following pair of norms.
\begin{equation*}
    \|\mathbf x_S\|_G^2 = \sum_{v\in S} \frac{1}{|\delta(v)|} x_v^2,
                  \;\textrm{and}\;
    \|\mathbf x_S\|_{G*}^2 = \sum_{v\in S} |\delta(v)| x_v^2.
\end{equation*}
\end{defn}

The messages from variables to factors are simple sums, similar to those in standard
belief propagation
\[
    \mu^{t+1}_{v\to F_i} = \frac{1}{|\delta(v)|}\sum_{F_j\in\delta(v)} \mu^t_{F_j\to v}.
\]
The factors always keep some vector on their base polyhedron, which at iteration
$t$ will be denoted by $\q^t_i\in B(F_i)$.
Then, based on the incoming messages, they update this vector by solving a convex problem, which is much cheaper than the exhaustive computation one has to
do for belief propagation (which is {\em exponential} in the factor size). We will denote the message sent from node $u$ to node $w$
at iteration $t$ by $\mu^t_{u\to w}$. If $\m^t_i\in\R^{V_i}$ is the \emph{vector} of messages
received at iteration $t$ at node $F_i$ (one message from each $v\in V_i$), then
the factor solves a projection problem parametrized by $(\m^t_i, \x^t_i)$, whose
solution is assigned to $\x^{t+1}_i$. Written formally, we have
\[
    \q^{t+1}_i = \underset{\q_i\in B(F_i)}{\mathrm{argmin}}\| \q_i - (\q_i^t - \m_i^t) \|^2_{G*}.
\]
As this is a convex separable problem on the base polytope, it can be solved for
example using the divide-and-conquer algorithm~\cite{fbach-ft}.
Having solved this problem, the factor sends the following messages to its
neighbours
\[
    \mu^{t+1}_{F_i\to v} = \q^{t+1}_v.
\]
Stated differently, it will send to every variable node $v$ the coordinate of the
stored vector corresponding to that variable.
At every iteration $t$ we can extract the current factorized approximation to the
full distribution by simply  considering the incoming messages at the variable nodes.
Specifically, the approximation $\q^t$ at time step has in the $v$-th coordinate
the sum of incoming messages at the node $v$, or formally
\[
    q^t_v = \sum_{F_i\in\delta(F_j)} \mu^t_{F_i\to v}.
\]
\looseness -1 Because the algorithm can be seen as performing block coordinate descent on a specific
problem (discussed in the appendix), the message passing algorithm described above possesses strong convergence
guarantees that depend on the structure of the factor graph. These guarantees even hold if all messages from nodes to factors, and all messages from factors to nodes are each computed {\em in parallel}.
An important quantity that appears in the convergence rate is the maximal variable connectivity
$\Delta_V = \max_{v\in V} |\delta(v)|$.
Based on recent new results by~\citet{nishihara2014convergence} on block coordinate descent for a similar dual (assuming that all factors depend on all variables, as considered by~\citet{jegelka2013reflection}), we extend their
analysis to obtain a linear convergence rate for our message passing scheme.
\begin{thm}[Extension of~\citet{nishihara2014convergence}]\label{thm:linconv}
    If the graph is $\Delta_V$-regular, s.t.\ every variable appears in exactly
    $\Delta_V$ factors, then the message passing algorithm converges linearly with rate
    ${(1-\frac{1}{|V|\Delta_V})}^2$. More specifically
    \begin{equation*}
        \|\q^t-\q^*\| \leq 2\|\q^0-\q^*\|_\infty\sqrt{\Delta_VE}
        (1-\frac{1}{|V|^2\Delta_V^2})^t,
    \end{equation*}
    where $\q^*$ is the optimal point, $\q^0$ is the initial point and
    $E$ is the number of edges in the factor graph.
\end{thm}

\makeatletter{}\section{Experiments}
\begin{figure*}[t]
    \centering
    \begin{subfigure}[h]{.155\textwidth}
        \centering
        \includegraphics[width=\textwidth]{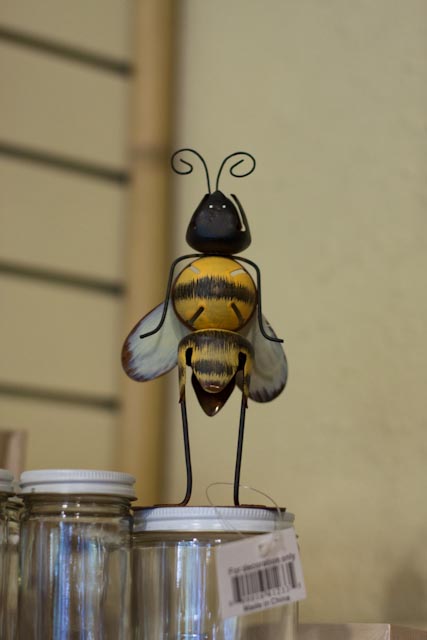}
        \caption{Original image.}
    \end{subfigure}
    \begin{subfigure}[h]{.155\textwidth}
        \centering
        \includegraphics[width=\textwidth]{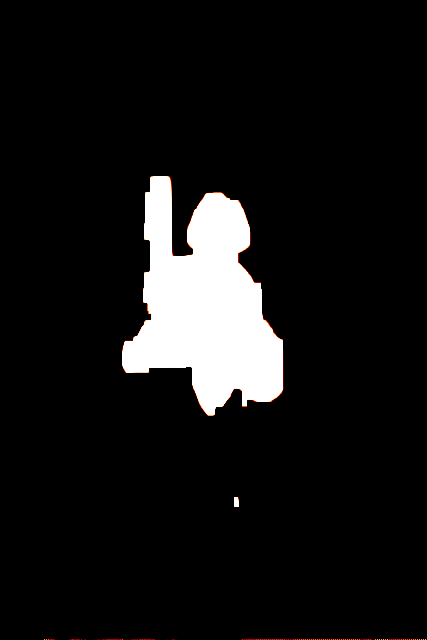}
        \caption{\textsc{FBP 1}.}
    \end{subfigure}
    \begin{subfigure}[h]{.155\textwidth}
        \centering
        \includegraphics[width=\textwidth]{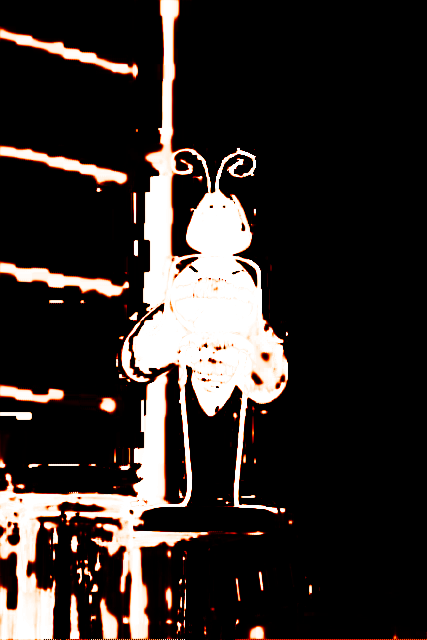}
        \caption{\textsc{FBP 2}.}
    \end{subfigure}
    \begin{subfigure}[h]{.155\textwidth}
        \centering
        \includegraphics[width=\textwidth]{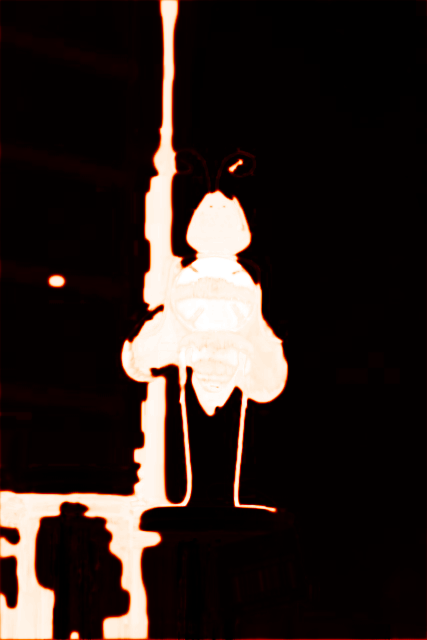}
        \caption{\textsc{FBP 3}.}
    \end{subfigure}
    \begin{subfigure}[h]{.155\textwidth}
        \centering
        \includegraphics[width=\textwidth]{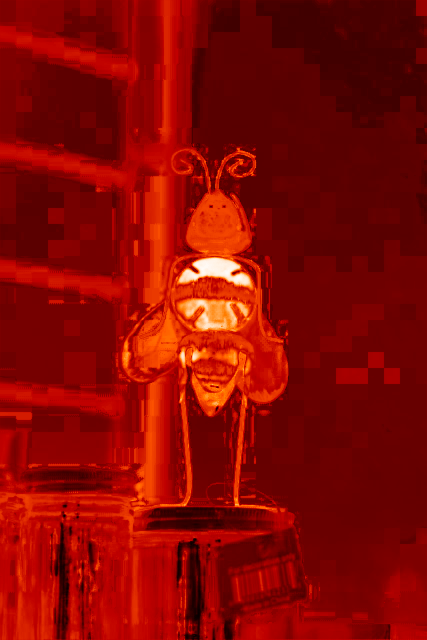}
        \caption{\textsc{FBP 4}.}
    \end{subfigure}
    \begin{subfigure}[h]{.155\textwidth}
        \centering
        \includegraphics[width=\textwidth]{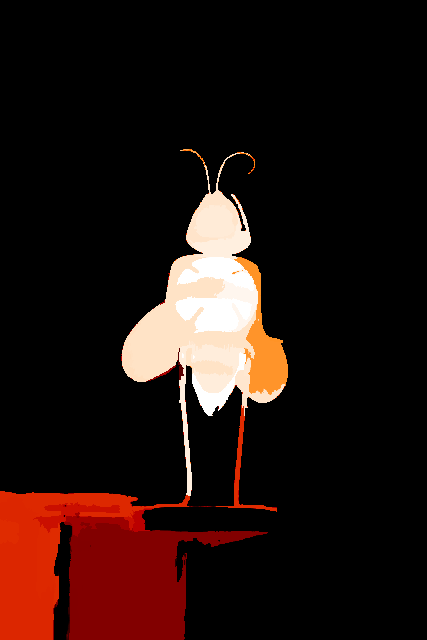}
        \caption{\textsc{HOP 1}.}
    \end{subfigure}

    \begin{subfigure}[h]{.155\textwidth}
        \centering
        \includegraphics[width=\textwidth]{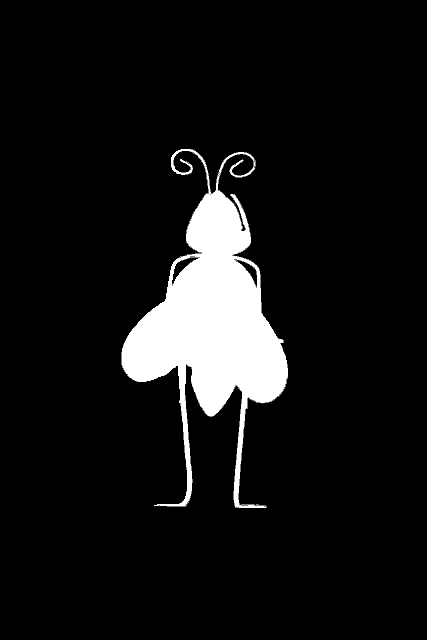}
        \caption{Ground truth.}
    \end{subfigure}
    \begin{subfigure}[h]{.155\textwidth}
        \centering
        \includegraphics[width=\textwidth]{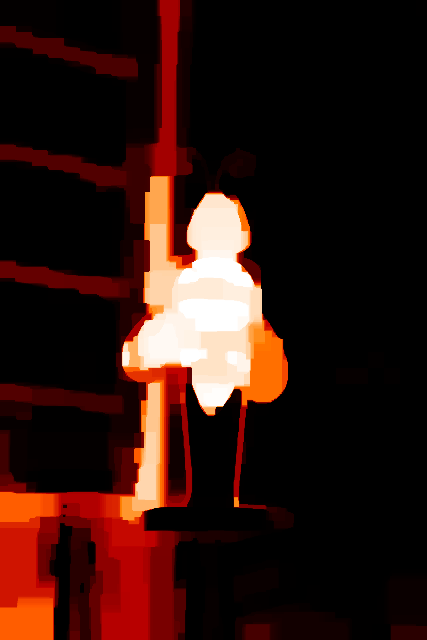}
        \caption{\textsc{DR 1}.}
    \end{subfigure}
    \begin{subfigure}[h]{.155\textwidth}
        \centering
        \includegraphics[width=\textwidth]{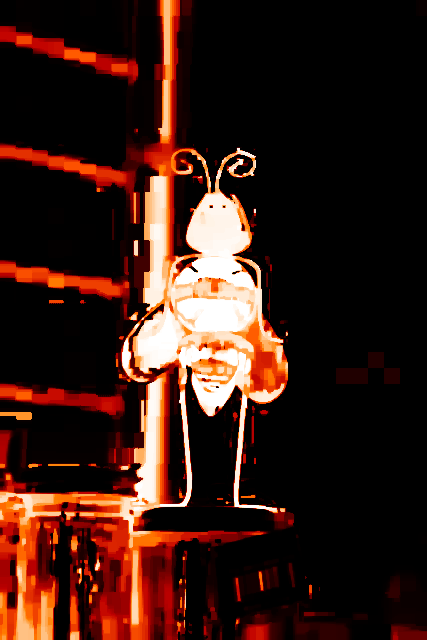}
        \caption{\textsc{DR 2}.}
    \end{subfigure}
    \begin{subfigure}[h]{.155\textwidth}
        \centering
        \includegraphics[width=\textwidth]{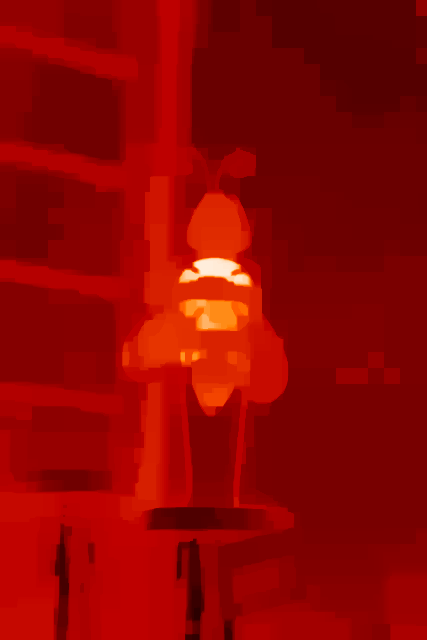}
        \caption{\textsc{DR 3}.}
    \end{subfigure}
    \begin{subfigure}[h]{.155\textwidth}
        \centering
        \includegraphics[width=\textwidth]{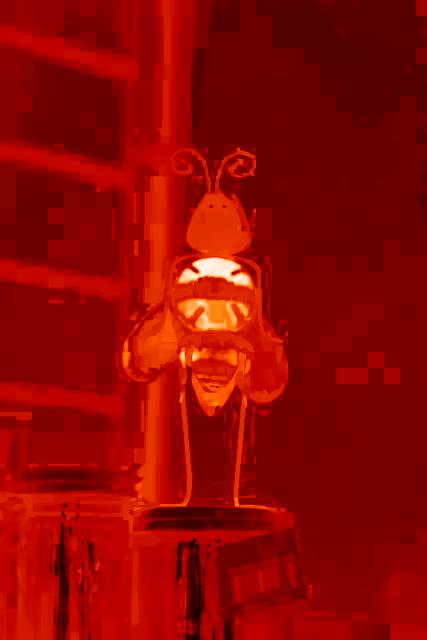}
        \caption{\textsc{DR 4}.}
    \end{subfigure}
    \begin{subfigure}[h]{.155\textwidth}
        \centering
        \includegraphics[width=\textwidth]{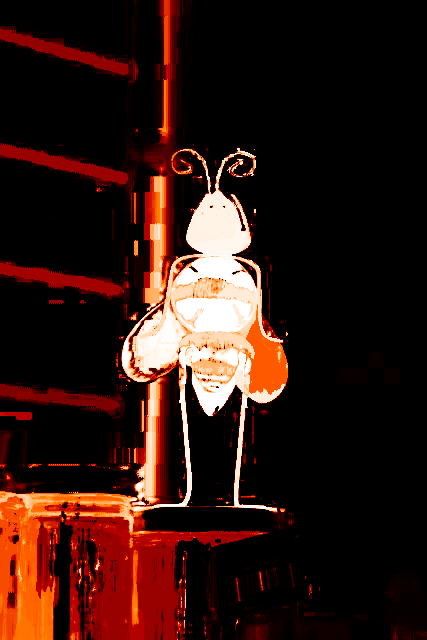}
        \caption{\textsc{HOP 2}.}
    \end{subfigure}
    \caption{Example marginals from the different approximation procedures for the original image (a) with ground truth segmentation (g).  For the results comparing \textsc{FBP} and \textsc{DR} (b-e,h-k) we have used the same
    pairwise weights and weights. The strength of the prior used monotonically decreases with the number specified after the respective method (i.e., b,f,h correspond to strong and e,k,l, to weak priors). Note how FBP is overconfident, whereas our method offers marginals with much higher dynamic range.\vspace{-3mm}}
    \label{fig:marginals}
\end{figure*}

We now report experimental results on applying our parallel variational inference scheme to a challenging image segmentation problem as motivated in \S\ref{sec:submodularity}. The goal of our experiments is to test the scalability of our approach to large problems, and to evaluate the quality of the marginals both qualitatively and quantitatively.
We used the data from~\citet{jegelka2011submodularity}, which contains a total of 36 images,
each with a highly detailed (pixel-level precision) ground truth segmentation. Due to intractability, we cannot
compute the exact marginals against which we would ideally wish to compare. As a proxy
for measuring the quality of the approximations, we use the area under the ROC curve (AUC) as
compared to the ground truth segmentation. We classify each pixel independently as fore- or background by comparing its approximate marginal against a threshold, which we vary to obtain the ROC curve.
We have used the following model,
which contains both pairwise and higher-order interactions.
\[
    F(A) = \alpha m(A) + \beta F_{\textrm{cut}}(A)
    + \gamma\sum_{P_i\in\mathcal P}\phi\left(\frac{|A\cap P_i|}{|P_i|}\right),
\]
where
\begin{itemize}
    \item the unary potentials $m(\cdot)$ were learned from labeled data using a
          5 component GMM;
      \item \looseness -1 the pairwise potentials $F_\textrm{cut}$ connect neighboring pixels $\x$ and
          $\x'$ with weights $w(\x, \x') = \exp(-\theta\|\x-\x\|^2)$, where $\x$ and
          $\x'$ are the RGB values of the pixels; 
    \item the higher order potentials were generated using the mean-shift algorithm
        of~\citet{comaniciu2002mean}. We have used two overlapping layers of superpixels, each layer with different granularity. The concave
        function was defined as $\phi(z) = z(1-z)$.
\end{itemize}
We compared the following inference techniques. The reported typical running times
are for an image of size 427x640 pixels on a quad core machine and we report the
wall clock time of the inference code (without setting up the factor graph or generating
the superpixels).
\begin{itemize}
    \item Unary potentials only with independent predictions, i.e., $\beta=\gamma=0$.
    \item Belief propagation (\textsc{BP}), mean-field (\textsc{MF}) and fractional
        belief propagation (\textsc{FBP}) for the pairwise model (i.e.\ $\gamma=0$).
        We have used the implementation from \texttt{libDAI}~\cite{mooij2010libdai}.
        The maximum number of iterations was set to 30. We note that this code is not parallelized.
        When we observe fast convergence, for example BP can converge in 3 iterations, it takes about
        45 seconds. Even though we have set a relatively low number of iterations, the running
        times can be extremely slow if the methods do not converge. For example, running mean-field
        for 30 iterations can take more than 3 minutes.
    \item \looseness -1 Our approach using only pairwise potentials ($\gamma=0$), solved using the
        total variation Douglas-Rachford (\textsc{DR}) code
        from~\cite{barbero11,barberoTV14,jegelka2013reflection}. We ran for at most
 100 iterations. The inference takes typically less
        than a second.
    \item Our approach with higher order potentials (\textsc{HOP}) only $(\beta=0)$.
        The inference takes less than 12 seconds.
\end{itemize}
For every method we tested several variants using different combinations for
$\alpha,\beta,\gamma$ and $\theta$ (exact numbers provided in the appendix). Then,
we performed a leave-one-out cross-validation for estimating the average AUC.
We have also generated a sequence of 10 trimaps by growing the boundary
around the true foreground to estimate accuracy over the hardest pixels, namely
those at the boundary.

\begin{figure}[!t]
    \centering
    \begin{subfigure}[h]{.5\textwidth}
        \begin{tabular}{lllll}
        \toprule
        Method & Avg.\ AUC & Std.\ Dev. & Avg.\ AUCT & Std.\ Dev. \\
        \midrule
        \textsc{hop} & \textbf{0.9638} & 0.0596  & \textbf{0.9567} & 0.0642 \\
        \textsc{dr} & 0.9602 & 0.0617 & 0.9496 & 0.0691 \\
        \textsc{fbp} & 0.9500 & 0.0662 & 0.9357 & 0.0975 \\
        \textsc{mf} & 0.9486 & 0.0675 & 0.9420 & 0.0764 \\
        \textsc{unary} & 0.9484 & 0.0681 & 0.9425 & 0.0759 \\
        \textsc{bp} & 0.9445 & 0.0779 & 0.9360 & 0.0964 \\
        \bottomrule
        \end{tabular}
    \end{subfigure}
    \caption{Average scores of the methods estimated using leave-one-out cross
    validation. The \emph{Avg.\ AUC} column is the average area under the ROC curve.
    The \emph{Avg.\ AUCT} column reports the average of the mean AUC over the 10 trimaps.
The second and the fourth columns are the standard deviation of the preceding columns.\vspace{-3mm}}
    \label{fig:table}
\end{figure}
\begin{figure*}[!htbp]
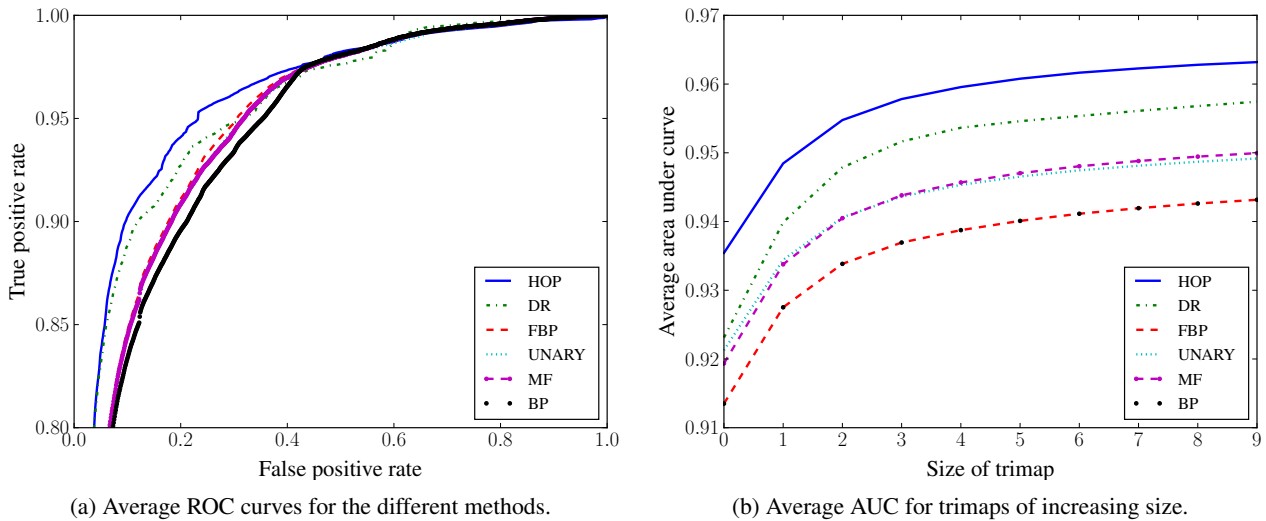

    \begin{subfigure}[h]{.5\textwidth}
        \centering
        \scalebox{0.45}{\input{auc.pgf}}
        \caption{Average ROC curves for the different methods.}
    \end{subfigure}
    \begin{subfigure}[h]{.5\textwidth}
        \centering
        \scalebox{0.45}{\input{tri.pgf}}
        \caption{Average AUC for trimaps of increasing size. }
    \end{subfigure}
    \caption{Comparison of inference methods in terms of their accuracy. For each method we optimize the parameters via grid-search, and report leave-one-out cross-validation results. (a) ROC curves for classifying pixels as fore- or background by independently thresholding marginals, averaged over the whole image. (b) Results over the trimaps (blurred boundaries around the ground truth segmentation), focusing on ``difficult'' pixels. For every algorithm and every of the 10 trimap sizes we report the average area under the curve.\vspace{-3mm}}
    \label{fig:curves}
\end{figure*}

\paragraph{Accuracy.} We first wish to quantitatively compare the accuracy of the approximate marginals. We report the aggregate results
in Figure~\ref{fig:table}, and the ROC curves in Figure~\ref{fig:curves}.
We can clearly see that our approach outperforms the traditional inference methods
for both objectives --- the AUC over the whole image and over the challenging boundary (trimaps). Sometimes we see very poor behavior of the alternative
methods, which can be attributed to either their over-confidence (as verified below), or the fact that they optimize non-convex objectives and can
fail to converge within the given number of iterations. Lastly, capturing high-order interactions leads to higher accuracy (in particular around the boundary) than pairwise potentials only.

\paragraph{Properties of marginals.}
We would also like to understand the qualitative characteristics of the resulting marginals of our methods when compared with the traditional techniques. From the discussion on the divergence minimization in \S\ref{sec:divergence}, we would expect the
approximate marginals to avoid assigning low probabilities and rather prefer to err conservatively, i.e., 
on the side of causing false positives. On the other hand, it is known that the results of
belief propagation are often over-confident. For this purpose, we provide a visual comparison in Figure~\ref{fig:marginals}. Namely, each of the four
$\textsc{FBP/DR}$ pairs are results using the respective algorithms for the same
parameters of the model. We observe exactly what the theory predicts --- the distribution obtained via \algo
is less concentrated around the object and  mass is spread around more. The contrast is starkest on Figures~\ref{fig:marginals} \emph{(b)} and \emph{(h)}, where we use a very strong pairwise prior (high $\beta$).
On Figures~\ref{fig:marginals} \emph{(e)} and \emph{(k)} we have used a very weak pairwise prior (low $\beta$), and as
expected the resulting marginals are mainly determined by the unary part and the
choice of inference procedure does not make a difference.
The results in the last column are from the higher order model, with two different
values of $\gamma$ (the strength of the higher order potential). We can see that
the resulting probabilities better preserve the boundaries of the object and the
fine details, which is one of the main benefits of using these models.
 
\makeatletter{}\section{Conclusion}
\looseness -1 We have addressed the problem of variational inference in log-supermodular distributions.  In particular, building on the \algo approach of \citet{djolonga14variational}, we established two natural, important interpretations of their method. First, we showed how \algo can be reduced to solving the well-studied minimum norm point problem, making a wealth of tools from submodular optimization suddenly available for approximate Bayesian inference. Secondly, we showed that the factorized distributions returned by \algo minimize a particular type of information divergence. Both of these theoretical connections are immediately algorithmically  useful. In particular, for the common case of decomposable models, both connections lead to efficient message passing algorithms. Exploiting the minimum norm connection, we proved strong convergence rates for a natural parallel approach, with convergence rates dependent on the factor graph structure. Lastly, we demonstrate our approach on a challenging image segmentation task. Our results demonstrate the accuracy of our marginals (in terms of AUC score) compared to those produced by classical techniques like belief propagation, mean field and variants, on models where these can be applied. We also show that performance can be further improved by moving to high-order potentials, leading to models where classical marginal inference techniques become intractable. We believe our results provide an important step towards practical, efficient inference in models with complex, high-order variable interactions.

\nocite{scikit-learn}

\bibliography{refs}
\bibliographystyle{icml2015}

\onecolumn
\appendix

\clearpage{}\makeatletter{}\section{Theory}
\subsection{Equivalence between the minimum norm and inference problems}

We will show a stronger result from which Theorem~\ref{thm:equiv} follows by taking
$w_i=1$ and $y_i=0$.

\begin{lemma}For positive weights $w_i>0$ the objectives
    $\sum_{i\in V} w_i(x_i - y_i)^2$ and
    $\sum_{i\in V} \frac{1}{w_i}\log(\exp(-w_i x_i) + \exp(-w_i y_i))$ have
    the same optimum under the constraint $\x\in B(F)$.
\end{lemma}
\begin{proof}
    For the second objective we have that
\[
    \sum_{i\in V} \frac{1}{w_i}\log(\exp(-w_i x_i) + \exp(-w_i y_i))
        \stackrel{\z=-\x}{=}
    \sum_{i\in V} \frac{1}{w_i}\log(\exp(w_i z_i) + \exp(-w_i y_i)).
\]
Hence, the optimum of the problem is the negative of the problem on $-B(F)$, which
is the base polytope of the submodular function $\overline{F}(A) = F(V-A) - F(V)$.
The gradient of this objective with respect to any $z_i$ is equal to
$\sigma(w_i(z_i + y_i))$, where $\sigma(u) = 1 /(1+e^{-u})$ is the sigmoid function.
As the weights $w_i$ are positive, we have that
\[
     \sigma(w_i(z_i + y_i)) \leq \sigma(w_j(z_j + y_j))
        \iff
     w_i (z_i + y_i) \leq w_j (z_j + y_j).
\]
From \citet{fujishige-book}[\emph{Theorem 8.1}], it follows that the above problem and $\sum_i w_i (z_i + y_i)^2$ have
the same solution $\z^*$ on $-B(F)=B(\overline{F})$. However, if $\z^*$ is the projection of $-\y$ onto $-B(F)$,
then $\x = -\z^*$ is the projection of $\y$ onto $B(F)$.
\end{proof}
\subsection{Connection to the infinite R\'enyi divergences}
If we expand the R\'enyi infinite divergence for $P(S)=\frac{1}{\Z_p}\exp(-F(S))$
and $Q(S)=\frac{1}{\Z_q}\exp(-q(S))$ for modular $q(\cdot)$ we get the following
\begin{equation}\label{eqn:dinfty}
    \Div{\infty}{P}{Q} = \log\sup_{A\seq V} \frac{P(A)}{Q(A)}
    =  \log\sup_{A\seq V} \frac{\exp(-F(A))/\Z_p}{\exp(-q(A))/\Z_q}
    = \log\Z_q  - \log\Z_p + \sup_{A\seq V} \{ q(A) - F(A) \}.
\end{equation}

As we will consider the problem of minimizing the above quantity with respect to the modular function
$q$ we can ignore the constant $\log\Z_p$. We will also expand the log-partition function of $q$,
$\log\Z_q = \sum_{i\in V} \log(1+\exp(-q_i))$ and introduce a new variable $t$ capturing the supremum above to arrive at the following
formulation.
\begin{align}\label{prob:infty}
\begin{split}
    \textrm{minimize}& \sum_{i\in V} \log(1+\exp(-q_i)) + t \\
    \textrm{subject to}&\quad q(A) \leq F(A) + t\quad \textrm{for all } A\seq V
\end{split}
\end{align}
\begin{defn}
    For any normalized submodular function $F$, define $\L^*(F)$ to be the optimum
    value of minimizing $\sum_{i\in V} \log(1+\exp(-s_i))$ subject to $\s\in B(F)$.
\end{defn}

To show the connection, we will need the following two lemmas.
\begin{lemma}[\citet{djolonga14variational}]\label{lem:duality}
    By strong Fenchel duality we have that
    \[
        \L^*(F) = \min_{\s\in B(F)} \sum_{i\in V} \log(1+\exp(-s_i)) = \sup_{\p\in[0,1]^V} \Ent[\p] - f(\p),
    \]
    where $f$ is the Lov\'asz extension of $F$ and $\Ent[\p]$ is the entropy of
    a random vector of independent Bernoulli random variables.
\end{lemma}
The following lemma is already known, as such functions have been already used, but we prove it for completeness.
\begin{lemma}\label{lem:fbeta}
    For any normalized submodular function $F:2^V\to\R$ define $F_\beta:2^V\to \mathbb{R}$
    as follows.
    \[
        F_\beta(S) =
        \begin{cases}
             0 & \text{if } S=\emptyset \\
             F(S) + \beta & \text{if } S\neq \emptyset
        \end{cases}
    \]
    Then, $F_\beta$ is submodular, normalized and has Lov\'asz extension
    $f_\beta(\w) = f(\w) + \beta\max_i w_i$.
\end{lemma}
\begin{proof}
    To show that $F$ is submodular we need to show that (see e.g.~\cite{fbach-ft}[Prop. 2.3]) for
    any $A\seq V$ and any $j,k\in V-A$ we have that
    $F_\beta(A\cup \{k\}) - F_\beta(A) \geq F_\beta(A\cup \{j, k\}) - F_\beta(A\cup \{j\})$. If
    $A\neq \emptyset$, then the above inequality follows immediately from the submodularity of $F$. Otherwise,
    we have the inequality
    $F(\{k\}) + \beta - F(\emptyset) \geq F(\{j, k\}) - F(\{j\})$,
    which again easily follows from the submodularity of $F$ and the fact that $\beta \geq 0$.
    The Lov\'asz extension of $F_\beta$ is defined as
    \[
        f_\beta(w) = \sum_{k=1}^{|V|} w_{j_k}\big( F_\beta(\{j_1,\ldots,j_k\}) - F_\beta(\{j_1,\ldots,j_{k-1}\}) \big),
    \]
    where the indices are chosen so that $w_{j_1}\geq w_{j_2}\geq\ldots\geq w_{j_{|V|}}$.
    Then, every term in the sum will be the same if we replace $F_\beta$ with $F$ (because $\beta$ will
    be added and subtracted), except for the first term when $k=1$. This term is equal to
    $w_{j_1}(F(\{j_1\}) + \beta - F(\emptyset))$, and the result follows immediately
    as $w_{j_1} = \max_i w_i$.
\end{proof}

\begin{lemma}
    For submodular functions Problem~\eqref{prob:infty} reaches the minimum for $t=0$.
\end{lemma}
\begin{proof}
Define $\OPT_\beta$ to be the optimum value of problem~\eqref{prob:infty} for
$t=\beta$. Note that $\OPT_\beta = \L^*(F_{\beta}) + \beta$. Then
\[
    \OPT_0 \stackrel{\mathrm{Lem.}\ref{lem:duality}}{=} \Ent[\p] - f(\p)
    \leq \Ent[\p] -f(\p) + \beta\underbrace{(1 - \max_i p_i)}_{\geq 0}
    \stackrel{\mathrm{Lem.}\ref{lem:fbeta}}{=} \Ent[\p] - f_\beta(\p) + \beta
    \stackrel{\mathrm{Lem.}\ref{lem:duality}}{\leq} \L^*(f_\beta) + \beta = \OPT_\beta.
\]
\end{proof}

And this immediately implies Lemma~\ref{lem:gb} and Theorem~\ref{thm:div-eq}.
We will now prove Lemma~\ref{lem:counterexample} by giving a specific counterexample.

\begin{lemma}
    There is a supermodular function for which the minimum is achieved for some $t > 0$.
\end{lemma}
\begin{proof}
    For $t=0$, the Lagrange dual problem of Problem~\eqref{prob:infty} is
\begin{align}
\begin{split}
    \textrm{maximize}& -\sum_{A\seq V} F(A)\lambda_A + \sum_{i\in V} h(\sum_{A\ni i} \lambda_A) \\
    \textrm{subject to}& \quad 0\leq\sum_{A\ni i} \lambda_A\leq 1 \quad \textrm{for all } i\in V, \\
\end{split}
\end{align}
where $h(p)=-p\log p - (1-p)\log(1-p)$ is the binary entropy function (defined so that $h(0)=h(1)=0$).
The Lagrange dual is easily derived if we note that $-h(-u)$ is the covex conjugate of the primal objective and use~\cite{boyd2004convex}[\S 5.1.6].
Consider the function \[ F(\emptyset)=0, F(\{1\}) = -20, F(\{2\}) = -8, F(\{1,2\}) = -16,\]
which is \emph{supermodular} as
\begin{align*}
    F(1 \mid \{2\}) &= -16 - (-8) = -8 > -20 =  F(\{1\}), \textrm{ and} \\
    F(2 \mid \{1\}) &= -16 - (-20) = 4 > 8 = F(\{2\}). \\
\end{align*}
For $t=1$ consider the primal feasible variable $x = (-19, -7)$, which has an objective
value $< 27.1$. For $t=0$ take the dual variable
$\lambda_\emptyset = \lambda_V = 0, \lambda_1 = 1 = \lambda_2 = 1$, with a
value of $-F(1)\lambda_1 - F(0)\lambda_0 + 0 = 28 > 27.1$ and thus
a strictly better value is achieved for $t=1$.
\end{proof}

\section{Proofs for Section 6}
We will first define a dual formulation of the minimization problem, for
which the claimed message passing scheme does BCD.
We will work with the set of all valid Lagrange multipliers
$\Lambda = \{ \lambda\in\mathbb R^{\sum_{i=1}^R |V_i|} \colon (\forall v\in V) \:
\sum_{i\in\delta(v)} \lambda_{v, i}=0 \}$
and the product of base polytopes
$\B=\otimes_{i=1}^R B(F_i)$. For any element $\y$ of either of these
sets we will denote by $y_{i, v}$ the coordinate corresponding
to variable $v\in V_i$ of the $i$-th block ($1\leq i\leq R$). Moreover,
for any vector $\x\in \R^V$ we will denote by $\x|_S$ its restriction to
the coordinates $S\seq V$.

\begin{thm}[Following~\cite{jegelka2013reflection}]\label{thm:new-dual}
    The dual problem of
    \begin{equation}
        \textrm{minimize}_\x f(\x) + \frac{1}{2}\|\x\|^2 =
        \textrm{minimize}_\x \sum_{i=1}^R (f(\x|_{V_i}) + \frac{1}{2}\|\x|_{V_i}\|_{G}^2)
    \end{equation}
    is equal to
    \begin{equation}\label{eqn:dual-problem}
        \underset{\lambda\in\Lambda, \y_i\in B(F_i)}{\textrm{maximize}}
        \sum_{i=1}^R-\frac{1}{2}\|\y_i-\lambda_i\|_{G*}^2,
    \end{equation}
\end{thm}
\begin{proof}The proof is based on that of~\cite{jegelka2013reflection}[Lemma 1].
The considered problem is the following.
\begin{align*}
\min_\x f(\x) + \frac{1}{2}\|\x|_2^2 &=
\min_\x \sum_{i=1}^R (f_i(\x_{V_i}) + \frac{1}{2}\|\x_{V_i}\|_G^2)
\;\textrm{s.t.}\;\x_{V_i} = \x|_{V_i}
\\
    &=
\min_{\x, \x_{V_i}} \max_{\lambda_i} \sum_{i=1}^R (f_i(\x_{V_i}) + \frac{1}{2}\|\x_{V_i}\|_G^2
    - \lambda_i^T(\x_{V_i} - \x|_{V_i})).
\end{align*}
Because we have zero duality gap, we can change the order of optimization.
Then, if we optimize for $\x$, we see that the Lagrange multipliers have to
belong to $\Lambda$, which was defined above. Hence, we have the following problem
\begin{align*}
    \max_{\lambda\in\Lambda} \sum_i \min_{\x_{V_i}} \max_{\y_i\in B(F_i)}
        (\x_{V_i}^T\y_i + \frac{1}{2}\|\x_{V_i}\|_G^2
            - \lambda_i^T\x_{V_i})
    = \max_{\lambda\in\Lambda} \sum_i\max_{\y_i\in B(F_i)}  \min_{\x_{V_i}}
        (\x_{V_i}^T\y_i + \frac{1}{2}\|\x_{V_i}\|_G^2
            - \lambda_i^T\x_{V_i}).
\end{align*}

Consider the inner problem, i.e.
\[  \minimize_{\x_{V_i}}
    \x_{V_i}^T\y_i + \frac{1}{2}\|\x_{V_i}\|_G^2 - \lambda_i^T\x_{V_i}
    =
    \minimize_{\x_{V_i}}
   \x_{V_i}^T(\y_i - \lambda_i) + \frac{1}{2}\|\x_{V_i}\|_G^2,
\]
which is exactly the negative of the convex conjugate of $\frac{1}{2}\|\cdot\|_{G}^2$ evaluated at
$\lambda_i - \y_i$. Because the convex conjugate is equal to
$\frac{1}{2}\|\cdot\|_{G*}^2$ (see e.g.\ \cite{boyd2004convex}[Ex.\ 3.22]),
the above minimum is equal to $-\frac{1}{2}\|\y_i-\lambda_i\|_{G*}^2$, which we
had to show.

\end{proof}

We can now easily see that the message passing algorithm is doing BCD for
the dual --- each node $F_i$ is first projecting to $\Lambda$
by subtracting the component-wise mean, and then clearly projecting onto $B(F_i)$
under the norm defined in Definition~\ref{defn:norms}.

We will now show the linear convergence rate. Because we consider the $k$-regular case, the
primal can be written in the following simpler form
\[
    \underset{\x}{\minimize} \sum_{i=1}^r f(\x|_{V_i}) + \frac{1}{2k}\|\x|_{V_i}\|^2,
\]
and the decomposed dual becomes the problem of finding the closest points
between $\Lambda$ and $\B$
\begin{equation}\label{eqn:dual-problem}
    \underset{\lambda\in\Lambda, \y_i\in B(F_i)}{\textrm{maximize}}
    \sum_i-\frac{k}{2}\|\y_i-\lambda_i\|^2.
\end{equation}

We now use exactly the same argument as in~\cite{nishihara2014convergence}[\S 3.3]
with some small changes necessary to accommodate our different definitions of
$\Lambda$ and $\B$. Please refer to that paper and references therein for the
terminology used in the remaining of the proof. We will show that the
Friedrich's angle between any two faces of $\B$ and $\Lambda$ is at most
$\frac{2}{k^2|V|^2}$, which combined with~\cite{nishihara2014convergence}[Thm.\ 2 and Cor. 5]
implies
the theorem. To make the notation easier to parse, we will assume that the
elements in $\B$ and $\Lambda$ are ordered so that first come the $|V_1|$
elements corresponding to $F_1$, then the $|V_2|$ elements corresponding to $F_2$
and so forth. Under this ordering, the vector space $\Lambda$ can be written as
the nullspace of the following matrix
\begin{equation}
    S = \frac{1}{\sqrt{k}} \bigg(\begin{array}{c|c|c}
        \underbrace{S_1}_{\in\R^{V\times V_1}}
            &\ldots &
            \underbrace{S_R}_{\in\R^{V\times V_R}} \end{array}\bigg),
        \text{where} \; [V_i]_{v, v'} = [v=v'].
\end{equation}
As noted in~\cite{nishihara2014convergence}, using~\cite{fbach-ft}[Prop.\ 4.7]
we can express the affine hull $\mathrm{aff}_0(\B_z)$ of any face $\B_z$ as
(where for each $i\in\{1,\ldots,R\}$ the sets $A_{r,1},\ldots,A_{r,M_r}$ form a
partition of $V_i$)
\[
    \textrm{aff}_0(\mathcal{B}_z) =
    \bigcap_{r=1}^R\bigcap_{m=1}^{M_r}
    \{ (\y_1,\ldots, \y_R) \colon y_r(A_{r,1}\cup\ldots\cup A_{r,m})=0 \},
    \textrm{where} \; \y_i\in B(F_i).
\]
This set can be also written as the nullspace of the following matrix
\[
T =
\begin{pmatrix}
    \frac{\mathbf{1}^T_{A_{1,1}}}{\sqrt{|A_{1,1}|}} &  &  \\
    \frac{\mathbf{1}^T_{A_{1,2}}}{\sqrt{|A_{1,2}|}} &  &  \\
    \vdots & &  \\
    \frac{\mathbf{1}^T_{A_{1,M_1}}}{\sqrt{|A_{1,M_1}|}} &  &  \\
    & \ddots & \\
    & & \frac{\mathbf{1}^T_{A_{1,1}}}{\sqrt{|A_{R,1}|}} \\
    & & \frac{\mathbf{1}^T_{A_{1,2}}}{\sqrt{|A_{R,2}|}} \\
    & & \vdots \\
    & & \frac{\mathbf{1}^T_{A_{1,M_R}}}{\sqrt{|A_{R,M_R}|}}
 \end{pmatrix}.
\]
To compute the Friedrich's angle we are interested in the singular values of
$ST^T$ \cite{nishihara2014convergence}[Lemma 6], which is equal to
\[
    ST^T = \frac{1}{\sqrt{k}}
    \bigg(
    \frac{\mathbf{1}_{A_{1,1}}}{\sqrt{|A_{1,1}|}},\ldots,
    \frac{\mathbf{1}_{1,M_1}}{\sqrt{|A_{1,M_1}|}},\ldots
    \frac{\mathbf{1}_{A_{R,1}}}{\sqrt{|A_{R,1}|}},\ldots,
    \frac{\mathbf{1}_{R,M_R}}{\sqrt{|A_{R,M_R}|}} \bigg).
\]
Hence, we have to analyze the eigenvalues of the square matrix $(ST^T)^T(ST^T)$,
whose rows and columns are indexed by
$\mathcal{I} = \{ (r,m) \colon r\in [R], m\in [M_r] \}$, and whose $((r_i,m_i), (r_j, m_j))$
entry is
\[
\frac{1}{k}\frac{|A_{r_i,m_i}\cap A_{r_j, m_j}|}
     {\sqrt{|A_{r_i,m_i}||A_{r_j, m_j}|}}.
\]
We create a graph with vertices $\mathcal{I}$ and we add edges between distinct
$(r_i, m_i)$ and $(r_j, m_j)$ with weight $|A_{r_i,m_i}\cap A_{r_j,m_j}|$
(zero weight means that we do not add that edge).  The normalized graph
Laplacian of this graph is equal to
\[
    [\mathcal{L}]_{(r_i,m_i), (r_j,m_j)} =
    \begin{cases}
        1 & \textrm{if } (r_i,m_i) = (r_j,m_j) \\
        - \frac{1}{k-1}\frac{|A_{r_i,m_i}\cap A_{r_j, m_j}|}
          {\sqrt{|A_{r_i,m_i}||A_{r_j, m_j}|}}
          & \textrm{otherwise}.
    \end{cases}
\]
Hence, $(ST^T)^T(ST^T)=I-\frac{k-1}{k}\mathcal L$.
We want to
lower-bound the Cheeger constant $h$ of this graph. Because the spectrum of a
graph is the union of the spectra of the connected components we will assume
that the graph is connected, as we can apply the same argument to every
component. From the definition of $h$ we have that
\[
    h \geq 2\frac{\textrm{minimum cut}}{\textrm{volume}},
    \textrm{where} \;
    \textrm{volume} =
    \sum_{v\in V} |\delta(v)|(|\delta(v)|-1) = |V|(k^2-k).
\]
What remains is to bound the minimum cut from below. Because the graph is
connected, for any cut $U$ there must exist some $v$ that is in sets on both
sides of the cut. Let $m$ be the number of sets in $U$ that contain it, and let
$k-m$ be the number of sets in the complement that contain it.
Then, the cut is of size is at least $m(k-m) \geq k-1$. Hence
\[
    h \geq 2\frac{k-1}{|V|(k^2-k)} = \frac{2}{|V|k}.
\]
And by Cheeger's inequality the smallest positive eigenvalue $\lambda_2$ of
the Laplacian $\mathcal L$ is at least $\frac{2}{k^2|V|^2}$, which from the
relationship $(ST^T)^T(ST^T)=I-\frac{k-1}{k}\mathcal L$ implies that the
squared Friedrich's angle $c_F^2$ between $\Lambda$ and $\B$ is at most
\[
    c_F^2\leq c_F= 1-\frac{k-1}{k}\lambda_2 \leq 1 - \frac{1}{|V|^2k^2},
\]
which completes the proof.

\section{Experiments}

We have used the parameter values in the table below.

\begin{figure}[h]
    \centering
\begin{tabular}{ll}
    Parameter & Values \\
    \toprule
    $\theta$ & $0.1, 0.001, 0.0001$ \\
    $\alpha$ & $1, 0.1, 0.01, 0.001$ \\
    $\beta$ & $10, 1, 0.1, 0.01, 0.001$ \\
    $\gamma$ & $10, 1, 0.1, 0.01, 0.001$ \\
    \bottomrule
\end{tabular}
\end{figure}
\clearpage{}

\end{document}